\documentclass{llncs}

\usepackage{times}
\usepackage{helvet}
\usepackage{courier}
\usepackage{color}
\usepackage{amstext, dsfont, amsmath, amssymb}
\usepackage{soul}

  \newtheorem{theorem*}{Theorem}

   \newcommand{\reals}{\mathbb{R}}		  
   
   \newcommand{\ex}{\mathop{\mathbb{E}}}
   
   \newcommand{\naturals}{\mathbb{N}}
   \newcommand{\N}{\mathbb{N}}
   
  \newcommand{\Fcal}{{\mathcal F}}

\newcommand{\new}[1]{{\color{blue} #1}}

\newcommand{\eps}{\epsilon}

\title{On a learning problem that is independent\\ of the set theory ZFC axioms} 
\author{Shai Ben-David  \and Pavel Hrube\v{s}  \and Shay Moran \and Amir Shpilka \and Amir Yehudayoff}
\institute{University of Waterloo \and Mathematical Institute of the Czech Academy of Science \and University of California, San Diego \and Tel Aviv University \and Technion-IIT}
\institute{}
\date{}

\begin{document}

\maketitle

\begin{abstract}
We consider the following statistical estimation problem:
given a family $\Fcal$ of real valued functions over some domain $X$ and an i.i.d. sample
drawn from an unknown distribution $P$ over $X$, find $h\in\Fcal$ such that
the expectation $\ex_P(h)$ is probably approximately equal to $\sup \{\ex_P(h): h \in \Fcal \}$. 
This Expectation Maximization (EMX) problem captures many well studied learning problems;
in fact, it is equivalent to Vapnik's general setting of learning.

\medskip
Surprisingly, we show that the EMX learnability, as well as the learning rates of some basic class $\Fcal$, depend on the cardinality of the continuum and is therefore independent of the set theory ZFC axioms (that are widely accepted as a formalization of the notion of a mathematical proof).

\medskip

We focus on the case where the functions in $\Fcal$ are Boolean,
which generalizes classification problems.
We study the interaction between the statistical sample complexity
of $\Fcal$ and its combinatorial structure. We introduce a new version of sample compression schemes and show that it characterizes EMX learnability for a wide family of classes. However, we show that for the class of finite subsets of the real line, the existence of such compression schemes is independent of set theory. We conclude that the learnability of that class with respect to the family of probability distributions of countable support is independent of the set theory ZFC axioms.

\medskip

We also explore the existence of 
a ``VC-dimension-like'' parameter that captures learnability in this setting.
Our results imply that that there exist no ``finitary" combinatorial parameter that characterizes EMX learnability in a way similar to the VC-dimension based characterization of binary valued classification problems. 



\medskip

\end{abstract}

\section{Introduction}

\label{sec:Intro}

A fundamental result of statistical learning theory is the characterization of PAC learnability in terms of the Vapnik-Chervonenkis dimension of a class \cite{vapnik2015uniform,blumer1989learnability}. This result provides tight upper and lower bounds on the worst-case sample complexity of learning a class 
of binary valued functions (in both the realizable and the agnostic settings of learning\footnote{The definitions of PAC learnability, sample complexity, the realizable and agnostic settings and VC-dimension, are basic notions of machine learning theory. See, for example chapters 3 and 6 of \cite{shai_shai_book}.}) in terms of a purely combinatorial parameter
(the VC-dimension). This characterization is remarkable in that it reduces a problem that concerns arbitrary probability distributions
and the existence of functions (``learners") of arbitrary complexity to the
``finitary'' notion of combinatorial shattering.


However, for some natural variants of that problem, such as multi-class classification
when the number of classes is large, 
we do not know of any parameter that provides such a characterization. 
This open problem has attracted considerable research effort 
(see e.g. \cite{daniely2011multiclass,daniely2012multiclass,daniely14optimal,daniely15multiclass,Bendavid:1995aa}).


We further explore the existence of such a parameter 
within a natural extension of multi-class classification, 
which we call the Expectation Maximization (EMX) problem, 
and is defined as follows (for a formal definition
see Section~\ref{sec:EMX}):
\begin{quote}
\emph{Given a family of functions ${\cal F}$ from some fixed domain $X$ to the real numbers and an unknown probability distribution $P$ over $X$,
find, based on a finite sample generated by $P$, a function in ${\cal F}$ 
whose expectation with respect to $P$ is (close to) maximal.}
\end{quote}

The EMX framework generalizes many well studied settings, such as 
classification, regression, as well as some clustering problems.
In fact, it  is equivalent to Vapnik's general framework 
of learning---see~\cite{Vapnik1998,vapnik1999overview} and references within.


{We focus on the setting when $\Fcal$ is a family Boolean valued functions (i.e.\ the range of each function is~$\{0,1\}$);
by identifying Boolean functions with the sets they indicate, 
the EMX problem in this case boils down to 
finding a set with maximal measure according to the unknown distribution.

This boolean setting corresponds to any learning problem within Vapnik's
setting that is modeled by zero/one loss;
in particular, any binary and multi-class classification problem.
We therefore find it natural 
to extend the exploration of a VC-dimension-like parameter to this setting.}

Our main conclusion is rather surprising. We show that for some class $\Fcal$,
the EMX learnability of that class, as well is its sample complexity rates when it is learnable, are determined
by the cardinality of the continuum
(i.e. the interval $[0,1]$). Consequently, deciding whether $\Fcal$ is EMX-learnable
is independent of set theory (the ZFC axioms). This follows by the well known independence theorems
of G\"{o}del~\cite{godel1940consistency} and Cohen~\cite{cohen1963independence,cohen1964independence}, see also \cite{Kunen,Jech}.

This result implies that there exist no combinatorial parameter 
of a finite character\footnote{We discuss this notion on a formal level 
in Section \ref{sec: dimension} later in the paper.} that characterizes EMX learnability the way the VC dimension and its variants characterize learnability in settings like binary label
prediction and real valued function learning.
Furthermore, our independence result applies already to ``weak learnability" --- the ability to find a function in the class that approximates the maximum possible expectation up to some additive constant, say $1/3$.

The independence of learnability result is shown 
for a specific problem | EMX learnability of 
the class of characteristic functions of finite subsets of the real line 
over the family of all countably supported probability distributions. While this case may not 
arise in practical ML applications, it does serve to show that the fundamental definitions of PAC learnability (in this case, their generalization to the EMX setting) is vulnerable in the sense of not being robust to changing the underlying set theoretical model.


\subsection{Technical contribution}
\subsubsection{Monotone compression schemes}
The main tool of our analysis is a novel variation of sample compression schemes
that we term \emph{monotone compression schemes}. 
Sample compression schemes were introduced 
by Littlestone and Warmuth~\cite{littlestone1986relating} 
in the context of binary classification.
Some natural learning algorithms, such as support vectors machines, can be viewed as 
implementing sample compression schemes , and
\cite{littlestone1986relating} showed that 
the existence of such schemes imply learnability.
It is also known that the reverse direction holds:
every learnable class can be learned
by a sample compression learning algorithm~\cite{moran2016sample,david2016supervised}.
We show that for classes satisfying a certain closure properties | union boundedness\footnote{{$\Fcal$ is union bounded if $ \forall h_1,h_2\in \Fcal \ \ \exists h \in\Fcal \ :~h_1\cup h_2\subseteq h$ (see Definition~\ref{def:unionbounded}).}} | existence of monotone compression is equivalent to EMX learnability.

\subsubsection{EMX-learnability and the cardinality of the continuum.}
The equivalence of EMX learnability with monotone compressions allows to translate that notion of learnability 
from the language of statistics to the language of combinatorics.
In particular, we consider the class of finitely supported Boolean functions on $X$,
and reduce the EMX learnability of this class to the following problem in
infinite combinatorics, which may be interesting in its own right.

\begin{definition}[The Finite Superset Reconstruction Game]\label{def: AB}
Let $X$ be a set. 
Consider the following collaborative two-players game:
Alice (``the compressor'') gets as input a finite set $S\subseteq X$. 
Then, she sends to Bob (``the reconstructor'') a subset
$S'\subseteq S$, according to a pre-agreed strategy.
Bob then outputs a finite set $\eta(S')\subseteq X$.
Their goal is to find a strategy for which $S\subseteq \eta(S')$ for every finite $S\subseteq X$.
\end{definition}

Alice can, of course, always send Bob $S'=S$
which he trivially reconstructs to $\eta(S)=S$.
We study the following question:
\begin{center}
{\em Can Alice send Bob subsets of bounded size?}
\end{center}

This depends on the cardinality of $X$.
For example, if $X$ is finite then Alice
can send the empty set $\emptyset\subseteq S$
which Bob reconstructs to $X$,  which is finite and clearly contains~$S$.
A more interesting example is when $X$
is countable, say $X=\N$.
In this case Alice can send the maximal element
in her input, $x_{max}=\max\{x:x\in S\}$,
which Bob reconstructs to $\{0,1,\ldots,x_{max}\}$, which contains $S$.

How about the case when $X$ is uncountable? 
We show that there is a strategy in which Alice
sends a subset of size at most $k$ if and only if $\lvert X\rvert < \aleph_k$.

Going back to EMX learnability, 
this implies that the class of finite subsets of the real unit interval is 
EMX learnable if and only if $2^{\aleph_0} < \aleph_{\omega}$. 
One should note that the statement $2^{\aleph_0} < \aleph_{\omega}$ 
is independent of the standard set theory (ZFC). 
In particular, it can neither be proven nor refuted in standard mathematics.

\subsubsection{Surprisingly fast agnostic learning rates.}

Another exotic phenomena
that we discovered concerns 
the agnostic EMX learning rates.
We
show that if $\Fcal$ is an EMX-learnable class 
that is {\emph{union bounded}}
then the dependence on the error parameter
$\eps$ in the agnostic-case sample complexity
is just $O(1/\eps)$ 
(see Theorems~\ref{thm:leaveoneout} and~\ref{thm:EMX_imply_compr} below). 
This is in contrast with the 
quadratic dependence of $\Theta(1/\eps^2)$,
which is standard in many agnostic learning problems.
Moreover, note that $\Theta(1/\eps^2)$ is also the
sample complexity of estimating $\ex[h]$ up to $\eps$ error,
and so this provides a non-trivial example where
\emph{finding} an approximate maximizer of $\sup_{h\in\Fcal}\ex[h]$
is done without obtaining a good estimate on the value of $\sup_{h\in\Fcal}\ex[h]$.

%

\subsection{Outline}
We start by introducing the EMX learning problem in Section \ref{sec:EMX}.
In Section \ref{sec: compression} we introduce monotone compression schemes and discuss some of their combinatorial properties. In Section \ref{sec: Leaningcompression} we relate the existence of monotone compression schemes to EMX learnability, and use that relationship to derive sample complexity bounds and a boosting result for that setting. 
In Section \ref{sec: continuum_compress} we show that the existence and size of monotone compression schemes for a certain class of finite sets (or their characteristic functions) is fully determined by the cardinality of the continuum. 
An immediate corollary to that analysis is the independence, with respect to ZFC set theory, of the EMX learnability of a certain class
(Corollary \ref{cor:EMX independence} there). 
Finally, Section \ref{sec: dimension} discusses the implications of our results to the existence of a combinatorial/finitary dimension characterizing EMX learnability.

\section{The expectation maximization problem (EMX)}
\label{sec:EMX}


Let $X$ be some domain set, and let ${\cal F}$ be a family of boolean 
 functions from $X$ to $\{0,1\}$.\footnote{The definition below
 can be extended to real valued functions.} 
 Given a sample $S$ of elements drawn i.i.d.\ from some unknown distribution $P$ over $X$, 
similarly to the definition of PAC learning, the EMX problem is about
finding with high probability
a function $f \in {\cal F}$ that approximately maximizes the expectation $\ex_P(f)$ with respect to $P$.
It is important to restrict the learning algorithm to being proper (i.e.\
outputting an element of ${\cal F}$),
since the all-ones functions is always a maximizer
of this expectation.

To make sense of $\ex_P(f)$, we need $f$ to be measurable
with respect to $P$.
To solve this measurability issue,
we make the following assumption:
\begin{quote}
{\em {\bf Assumption.} 
All distributions in this text are countably supported 
over the $\sigma$-algebra of all subsets of $X$.}
\end{quote}
As discussed in the introduction,
in our opinion,
this assumption does not harm the main message of this text
(but is necessary for the correctness of our proofs).

Let $Opt_P(\Fcal)$ denote $\sup_{h \in \Fcal}\ex_P(h)$.

\begin{definition} [The Expectation Maximization Problem (EMX)] \label{def:EMX}
A function 
$$G: \bigcup_{i \in \naturals} X^i \to \Fcal$$ is an $(\epsilon, \delta)$-\emph{EMX-learner}
for a class of functions  $\Fcal$ if for some integer $m = m(\epsilon,\delta)$,
\[\Pr_{S \sim P^m}\bigl[Opt_P(\Fcal) - \ex_P\bigl(G(S)\bigr) \geq \epsilon\bigr] \leq \delta\,, \] 
for every {(countably supported)} probability distribution 
$P$ over $X$.

We say that $\Fcal$ is \emph{EMX-learnable} if for every 
$\epsilon ,\delta >0$ there exists an $(\epsilon, \delta)$-EMX-learner for $\Fcal$ .
\end{definition}

The EMX problem is essentially equivalent to 
{Vapnik's general setting of learning; The definitions are syntactically different since 
Vapnik considers supervised learning problems with labeled examples.
Also, Vapnik focuses continuous domains, whereas EMX is phrased in a more abstract ``set theoretic'' settings.}

{This formulation of the EMX problem} was (implicitly) introduced in \cite{shalev_alt/2011} in the context of the task of proper learning when the labeling rule is known to the learner.
Many common statistical learning tasks can be naturally cast as particular instances of EMX,
{in the more general case that $\Fcal \subseteq {\mathbb R}^X$ and under more general measurability assumptions}. 
In particular problems that can be cast as ``generalized loss minimization" as defined in Chapter 3 of~\cite{shai_shai_book}, including:
\begin{itemize}
\item Binary classification prediction in the proper setting. 
\item Multi-class prediction in the proper setting.
 \item $K$-center clustering problems (in the statistical setting).
\item Linear regression.
\item Proper learning when the labels are known.
\item Statistical loss minimization.
\end{itemize}

\section{Monotone Compression Schemes} \label{sec: compression}
For $n\in\naturals$, let $X^{\leq n}$ denote the set of all 
sequences  (or samples) of size at most $n$ whose elements belong to $X$. 
Let $\Fcal \subseteq \{0,1\}^X$.\footnote{{The definition can be
naturally generalized to $\Fcal \subseteq Y^X$
for any ordered set $Y$.}}

\begin{definition}[Monotone Compression Schemes] 
\label{def:mon_comp} 
For $m,d \in \naturals$, 
an $``m\to d$ monotone compression scheme for $\Fcal$''
is a function $\eta: X^{\leq d} \to \Fcal$ such that

\begin{quote} 
%
%

for every every  $m' \leq m$, every $h \in \Fcal$ and every $x_1, \ldots, x_{m'} \in h$,

there exist $i_1, \ldots i_k$ for some $k \leq d$ so that

for all $i \leq m'$ $$x_i \in \eta[(x_{i_1}, \ldots x_{i_k})].$$
\end{quote}

The function $\eta$ is called the decompression or the reconstruction function.
%
%
\end{definition}

{The intuition is that after observing $x_1,\ldots,x_m$ that belong to some unknown set $h \in \Fcal$
there is a way to compress them to $x_{i_1},\ldots,x_{i_k}$
so that the reconstruction $\eta(x_{i_1},\ldots,x_{i_k})$
contains all the observed examples.
In other words,
the finite superset reconstruction game can be solved
using an $m \to d$ monotone compression scheme
where Alice gets a set of size $m$
and sends to Bob a set of size $\leq d$.}



\paragraph{Side-information.}
It will sometimes be convenient to describe the reconstruction function $\eta$
as if it also receives a finite number of bits as its input (like in the lemma below). 
This can be simulated within the above definition,
e.g.\ by repetitions of elements or by utilizing the ordering of the sequence.

\subsection*{Uniformity}

We say that $\Fcal$ has a (non-uniform) monotone compression scheme of size $d$ if for every $m \in \naturals$, 
there is an  $m\to d$ monotone compression scheme for $\Fcal$.
 We say that $\Fcal$ has a {\em uniform} monotone  compression scheme of size $d$ if there is a single function $\eta: X^{\leq d} \to \Fcal$ that is an $m\to d$ monotone compression scheme for $\Fcal$, for every $m \in \naturals$.

The following lemma shows that the difference between non-uniform and uniform
monotone compression schemes is negligible. 
Nevertheless, 
whenever we assume that $\Fcal$ has monotone compression scheme of size $d$
we mean the weaker assumption, where the scheme is non-uniform.
\begin{lemma} \label{lem:non_to_uniform}
If a class $\Fcal$ has a non-uniform monotone compression scheme of some size $d$, 
then for every monotone $f: \naturals \to \naturals$
such that $\lim_{m \to \infty}f(m)= \infty$, 
there is a uniform monotone compression scheme for $\Fcal$
that compresses samples of size $m$ 
to subsample of size $d$ plus $f^{-1}(m)$ extra bits of side information.
\end{lemma}

\begin{proof} By assumption, for every $m \in \naturals$ there is  a monotone (de-)compression function $\eta_m : X^{\leq d} \to \Fcal$.
Given any $S= \{x_1, \ldots, x_m\}$, let $m'$ be such that $f(m') \geq m$, find $S'= (x_1, \ldots x_k)$ such that $k \leq d$ and $\eta_{f(m')}(S') \supset S$ and define a decompression function $\eta$ by  $\eta( S', m')=\eta_{f(m')}(S')$.
\qed
\end{proof}



\subsection*{Monotone compression schemes versus sample compression schemes}
It is natural to ask how monotone sample compression schemes and classical (Littlestone-Warmuth)
sample compression schemes relate to each other.

In one direction, there are classes for which there is a monotone compression scheme
of size 0 but there is no sample compression scheme of a constant size.
One example of such a class is $\{0,1\}^{\N}$. 
A monotone compression scheme of size $0$
for this class is obtained by setting $\eta(\emptyset)=\N$.
However, this class does not have a sample compression scheme
of a constant size since it is not PAC learnable.

The other direction | whether sample compression scheme
of constant size implies a monotone compression scheme of a constant size? |
remains open.
A recent result of~\cite{moran2016sample} implies that
a class $\Fcal\subseteq\{0,1\}^X$ has a sample compression
scheme of a finite size if and only if it has a finite VC dimension.
Thus, an equivalent formulation is whether
every class with a finite VC dimension has a monotone compression
scheme of a finite size.

A related open problem is whether every class with finite VC dimension has
a \emph{proper}\footnote{A sample compression scheme for $\Fcal$ is called proper
if the range of the reconstruction $\eta$ is $\Fcal$.} sample compression scheme.
In fact, there is a tight relationship between proper compression schemes and monotone compression schemes;

\begin{claim}[``monotone versus proper compression schemes'']
\begin{enumerate}
\item {\bf{proper$\implies$ monotone}.}
Every class that has a proper compression scheme to some size (which may depend on the sample size too)
has a monotone compression scheme to the same size.
\item {{\bf{monotone $\implies$ proper}.} }
For every class $H$ there exists a class $H'$ such the $VCdim(H)=VCdim(H')$ and every monotone compression scheme 
for $H'$ can be turned into a proper compression scheme for $H$ to the same size.
\end{enumerate}
\end{claim} 
\begin{proof}
Part 1 follows trivially from the definitions of those two notions of compression.
For the second part, given $\Fcal\subseteq \{0,1\}^X$,
identify each concept $h\in \Fcal$ with the set $S_h=\bigl\{\bigl(x,h(x)\bigr) : x\in X\bigr\}$. 
Let $S_{\Fcal} = \{S_h : h\in H\}$.
One can verify that $S_{\Fcal}$ has the same VC dimension like $\Fcal$,
and that a monotone compression scheme for $S_{\Fcal}$ corresponds
to a proper sample compression scheme for $\Fcal$.
\end{proof}

Thus, for VC classes, the existence of monotone compression schemes
and proper compression schemes are equivalent:
\begin{corollary}
For every $d ,k\in \naturals$, the following statements are equivalent.
\begin{itemize}
\item Every class with VC dimension $d$ has a monotone compression scheme of size $k$.
\item Every class with VC dimension $d$ has a proper compression scheme of size $k$.
\end{itemize}
\end{corollary}
Note that if $H$ is a class of infinite VC dimension, since it has no non-trivial sample compression scheme, it has no proper one. 
Applying the construction of the Claim above, one gets that the corresponding class, $S_{\Fcal} = \{S_h : h\in H\}$, has no monotone compression scheme. We therefore get,
\begin{corollary}
There exists a class $\Fcal$ such that for any $k\in\mathbb{N}$,
the class $\Fcal$ does not have a monotone compression scheme of size~$k$.
\end{corollary}

\section{EMX learnability and monotone compression schemes} \label{sec: Leaningcompression}

In this section we relate monotone compression schemes with EMX learnability.
Theorem~\ref{thm:compr_imply_EMX} and Theorem~\ref{thm:leaveoneout} show that
the existence of monotone compression schemes imply EMX learnability,
and Theorem~\ref{thm:EMX_imply_compr} shows the reverse implication holds whenever $\Fcal$
is union bounded (see Definition~\ref{def:unionbounded}).


\begin{theorem}
\label{thm:compr_imply_EMX}
Let $\Fcal \subseteq [0,1]^X$. Assume there is a monotone compression scheme $\eta$
for $\Fcal$ of size $k\in\mathbb{N}$. Then $\Fcal$ is EMX-learnable
with sample complexity 
\[m(\eps,\delta) = O\Bigl(\frac{k\log(k/\eps) + \log(1/\delta)}{\eps^2}\Bigr).\]
\end{theorem}

The proof of Theorem~\ref{thm:compr_imply_EMX} 
follows from the standard ``compression $\implies$ generalization'' argument of Littlestone and Warmuth~\cite{littlestone1986relating}. 
For completeness we include it in the appendix.

If $\Fcal$ satisfies the following closure property,  then the dependence on $\eps$
in Theorem~\ref{thm:compr_imply_EMX} can be improved from $1/\eps^2$ to $1/\eps$.

\begin{definition}[Union Bounded]\label{def:unionbounded}
  We say that a family of sets $\Fcal$ is \emph{union bounded} if for every $h_1, h_2 \in \Fcal$ there exists some $h_3 \in \Fcal$ such that $h_1 \cup h_2 \subseteq h_3$.
\end{definition} 
Note that every class that is closed under finite unions is also union bounded. However, the latter condition is 
satisfied by many natural classes that are not closed under unions, such as the class of all axis aligned rectangles in $\reals^d$ or the class of all convex polygons. 


\begin{theorem}\label{thm:leaveoneout}
Assume $\Fcal$ is union bounded and that it has a monotone compression of size $d$.
Then, there is a learning function $A$ such that for every distribution~$P$ 
and~$m\in \naturals$:
\[\ex_{S \sim P^m}\Bigl[\sup_{h\in H}\bigl\{\ex_P[h]\bigr\} - \ex_P[h_{S}]\Bigr]
\leq
\frac{d}{m+1},\]
where $h_{S}=A(S) \in \Fcal$ is the output of $A$ on input $S$.
\end{theorem}
Note that via Markov inequality this yields an $(\eps,\delta)$-sample complexity of $O(\frac{1}{\eps\delta})$.
Thus, the dependence on $\delta$ is worse then the logarithmic dependence of the standard generalization bound, but the dependence on $\eps$ is better.

\begin{proof}
Let $\eta$ be the decompression\footnote{If the compression scheme is non-uniform then
set $\eta$ as the decompression function of the $(m+1)\to d$ monotone compression scheme.} 
function of $\Fcal$.
We first describe the algorithm $A$. For every sample $S$, let $A(S)$ be a member of $\Fcal$ such that 
\[A(S) \supset \bigcup_{S'\subseteq S, \lvert S'\rvert \leq d}\eta[S'].\]
Since $\Fcal$ is union bounded such $A(S)$ exists
for every $S$.

Let $T=(x_1,\ldots, x_{m+1})$ be drawn from $P^{m+1}$.
What we need to show can be restated as:
\begin{equation}\label{eq:1}
\forall h\in \Fcal:~~~
\Pr_{T \sim P^{m+1}}\bigl[x_{m+1}\in h\bigr]
- 
\Pr_{T \sim P^{m+1}}\bigl[x_{m+1}\in A(T^{\overline{m+1}})\bigr]
\leq 
\frac{d}{m+1} .
\end{equation}
where $T^{\overline{m+1}} = (x_1,\ldots,x_m)$
and more generally $T^{\overline{i}} = (x_1.\ldots,x_{i-1},x_{i+1},\ldots,x_{m+1})$.

We prove Equation~\eqref{eq:1} by a leave--one-out symmetrization argument,
which is derived by the following sampling process.
First sample $T = (x_1,\ldots,x_{m+1})\sim P^{m+1}$, and then 
independently pick $i\in [m+1]$ uniformly at random.
Clearly, it holds that
\begin{align*}
\Pr_{ T \sim P^{m+1}}\Bigl[x_{m+1}\in A(T^{\overline{m+1}})\Bigr]
&=
\Pr_{T\sim P^{m+1}, i\sim[m+1] }\Bigl[x_{i}\in A(T^{\overline{i}})\Bigr]\\
&\text{and}\\
\forall h\in \Fcal:~~~\Pr_{T\sim P^{m+1}}\bigl[x_{m+1}\in h\bigr]
&=
\Pr_{T\sim P^{m+1}, i\sim[m+1] }\bigl[x_{i}\in h\bigr].
\end{align*}
Thus, it suffices to prove the following stronger statement:
\[
\bigl(\forall h\in \Fcal\bigr)\bigl(\forall T\in X^{m+1}\bigr):
\Pr_{i \sim [m+1]}\bigl[x_{i}\in h\bigr]
- 
\Pr_{i\sim [m+1]}\bigl[x_{i}\in A( T^{\overline{i}})\bigr]
\leq 
\frac{d}{m+1}.
\]
Let $h\in \Fcal$.
Assume without loss of generality that $x_1,\ldots, x_d$ are the $d$ 
indices that satisfy $h(x_i)\leq \eta[x_1,\ldots,x_d](x_i)$
for every $i\in [m+1]$.
Thus, for every $i > d$:
\begin{align*}
h(x_i) \leq \eta[x_1,\ldots,x_d](x_i)
      \leq  A(T^{\overline{i}})(x_i). 
\end{align*}

Therefore,
\begin{align*}
\Pr_{i\sim[m+1]}\bigl[x_i\in h\bigr]  - \Pr_{i\sim[m+1]}\bigl[x_i\in A(T^{\overline{i}})\bigr] 
&\leq 
\Pr_{i\sim[m+1]}\Bigl[h(x_i)=1 \land \bigl(A(T^{\bar i})\bigr)(x_i)=0\Bigr]\\
&\leq 
\Pr_{i\sim [m+1]}[i \leq d] = 
\frac{d}{m+1}
\end{align*}
\qed
\end{proof}

\begin{theorem}[learnability implies non-uniform compressibility]\label{thm:EMX_imply_compr}
Assume $\Fcal$ is union bounded and $(\epsilon=1/3, \delta=1/3)$-EMX learnable 
by a learning function $G$ with sample size $d_0=m(1/3, 1/3)$.
Then for every $m\in\N$ there is an $m \to  3d_0/2$ monotone compression scheme for $\Fcal$.
Moreover, it is enough to assume that $G$ is a learner with respect to
 distributions with a finite support. 
\end{theorem}
This theorem yields a non-uniform monotone compression in the sense that
for different sample sizes we use different reconstruction functions. 
Note however that the generalization bounds for compression schemes
remain valid for non-uniform schemes with the same proof.

We can now apply Lemma \ref{lem:non_to_uniform} to deduce the existence of a uniform monotone sample compression.

\begin{corollary}[learnability implies uniform compressibility]\label{cor:EMX_imply_compr}
Let $\Fcal$ and $d_0$ be like in Theorem~\ref{thm:EMX_imply_compr}.
Then for every non--decreasing $f:\N\to\N$ such that $\lim_{n\to\infty}f(n)=\infty$
there is a monotone compression scheme for $\Fcal$ of size at most $3 d_0 /2$
plus $f^{-1}(m)$ bits,
where $m$ is the size of the input sample.
\end{corollary}

Note that we may pick any $f$ that tends to infinity arbitrarily fast
(e.g.\ $f^{-1}(m)=\log^*(m)$).

\begin{proof}[Theorem~\ref{thm:EMX_imply_compr}]
First, note that if $G$ is an $(\epsilon=1/3, \delta=1/3)$-EMX learner for $\Fcal$ and $G'$ is an EMX learner such that for all $S \subseteq X$, $G'(S) \supset G(S)$, then $G'$ is also an $(\epsilon=1/3, \delta=1/3)$-EMX learner for $\Fcal$.  Furthermore, we may assume that $\bigcup \Fcal=X$. It follows now, by the union boundedness of $\Fcal$, that without loss of generality we may assume that for any $S$, $G(S) \supset S$.

Pick some $h \in \Fcal$ and let $S = (x_1,\ldots,x_{m'})$ with $m'\leq m$ 
so that $x_i \in h$ for each $i$. 
We would like to exhibit a reconstruction function $\eta$,
and argue that there is a subsample $S'\subseteq S$ of size at most $3d_0/2$
such that $\eta[S']\supseteq S$.
To this end we first extend $G$ to samples of size larger than $d_0$ by setting
\[E'(S) = \bigcup_{S'\subseteq S, \lvert S'\rvert \leq d_0}G(S').\]
Apply the closure property of $\Fcal$ to pick  some $E(S)\in \Fcal$ such that $E(S) \supset E'(S)$.

The following process describes the compression (i.e.\ the subsample $S'$).
\begin{quote}
Set $S_0=S$.\\
If there is $x\in S_i$ such that\footnote{Subtraction as sequences.} $x\in E(S_{i} \setminus\{x\})$ then
pick such an $x$ and set $S_{i+1} = S_i\setminus\{x\}$.\\
Else, return $S_i$.
\end{quote}
Note that, by construction,  for every $i$ $S_{i-1} \subseteq E(S_i)$. Furthermore, for any $S' \supset S_i$, $S_{i-1} \subseteq E(S')$.
Since $G$ is an $(\epsilon=1/3, \delta=1/3)$-learner, 
it follows that the above process will proceed 
for at least as long as $\lvert S_i\rvert > 3 d_0 /2$. 
Indeed, assume towards a contradiction that $2|S_i|/3 >  d_0$ and for every $x \in S_i$
we have $x \not \in E(S_i \setminus \{x\})$.
This means that each $x \in S_i$ is not contained in any output of $G$
on $d_0$ samples from $S_i$ that do not contain $x$.
Let $P$ be the uniform distribution on $S_i$.
Thus, $\sup \{ \ex_P(f) : f\in H\} = 1$ but for every $S \in X^{d_0}$ we have
$\ex_P(G(S)) < 2/3$.

Let $S'$ be the output of the above process.
The reconstruction acts on $S'$ by applying $E$ on it $m$ times:
\[\eta[S']=E\bigr(E(\ldots E(S')\bigl).\] 
By construction of $S'$ and monotonicity of $G$ it follows that $\eta[S']\supseteq S$
as required.

\qed
\end{proof}

\begin{corollary}
\label{cor:main}
The following statements are equivalent
for a union bounded class $\Fcal$:
\begin{description}
\item[-- Learnability:] $\Fcal$ is EMX-learnable.
\item[-- Weak learnability:] $\Fcal$ is weakly EMX-learnable.
\item[-- Weak compressibility:] There exists an $(m+1) \to m$  monotone compression scheme for $\Fcal$ for some $m\in\N$.
\item[-- Strong non-uniform compressibility:] There is $d\in\N$ such that for every $m\in\N$ there is an $m \to d$ monotone
compression scheme for $\Fcal$.
\item[-- Strong uniform compressibility:] There is $d\in\N$ such that for every $f:\N\to\N$ with $\lim_{n\to\infty}f(n)=\infty$,
there is an $m \to d$  monotone compression scheme for $\Fcal$ that uses plus $f^{-1}(m)$ bits of side information.
\end{description}

\end{corollary}

\section{Monotone compressions for classes of finite sets} \label{sec: continuum_compress}

In this section we analyze monotone compression for classes of the form
 $$\Fcal^X_{fin} = \{h \subseteq X : h ~\mbox{is a finite set} \}.$$
 We show that such a class has a uniform monotone compression scheme to size $k+1$ if and only if $|X| \leq \aleph_k$. Uncoiling Definition~\ref{def:mon_comp}, a uniform $k$-size monotone compression scheme for  $\Fcal^X_{fin}$ is the following:
a function $\eta: X^{\leq k}\to \Fcal^X_{fin}$ such that, for every $x_1,\ldots,x_m \in X$ there exists $x_{i_1},\ldots,x_{i_{\ell}}$ where $\ell\leq k$ with 
 \[\{x_1,\ldots,x_m\}\subseteq \eta(x_{i_1},\ldots,x_{i_\ell}).\]
 In terms of Definition~\ref{def: AB}, Alice compresses $S=\{x_1,\ldots,x_m\}$ to $S'=\{x_{i_1},\ldots,x_{i_\ell}\}$ and Bob reconstructs $\eta(S')$. 
 
  \begin{theorem} \label{thm_compressionUB}
 For every $k \in \naturals$ and every domain set $X$ of cardinality $\leq \aleph_k$, the class $\Fcal^X_{fin}$ has a $(k+1)$-size monotone compression scheme.
 
 \end{theorem}
 
 \begin{proof}
 The proof is by induction on $k$. For $k=0$, let $\prec$ be an ordering of $X$ of order type $\omega$ (that is, isomorphic to the ordering of the natural numbers). Given a finite $S \subseteq X$ compress it to its $\prec$ - maximal element. The decompression function is simply $\eta[x] = \{y: y \preceq x\}$.
 
 Assume the claim holds for all $k' < k$ and let $X$ be a set of cardinality $\aleph_k$. Let $\prec_k$ be an ordering of $X$ of order type $\omega_k$ (that is, a well ordering in which every proper initial segment has cardinality $< \aleph_k$). For every $x \in X$, let $k(x) < k$ be such that $|\{y: y \prec_k x\}|=\aleph_{k(x)}$.

 Given a finite $S \subseteq X$ let $x$ be its $\prec_k$ maximal element. Note that $S \setminus \{x\} \subseteq \{y: y \prec_k x\}$.
 By the induction hypothesis there exists a monotone compression scheme of size $k(x)+1$ for $\Fcal^{\{y: y \prec_k x\}}_{fin}$. Compress  $S \setminus \{x\}$ using that scheme and add $x$ to the compressed set. The decompression uses $x$ to determine the domain subset $\{y: y \prec_k x\}$ and then applies the compression of that subset of size $\leq k$  to the remainder of the compressed image of $S$.
 
 \end{proof}
 \begin{theorem} \label{thm_compressionLB}
 For every $k \in \naturals$ and every domain set $X$ if the class $\Fcal^X_{fin}$ has a $(k+1)$-size monotone compression scheme then the cardinality of $X$ is $\leq \aleph_k$. Furthermore, the result already follows from just being able to monotonically compress samples of size $k+2$ to subsamples of size $\leq k+1$.

 \end{theorem}
  \begin{proof}
 We prove this theorem through the following lemma.
 
 \begin{lemma}
 \label{lem:DecreaseSize}
 Let $k \in \naturals$ and 
let $X' \subset X$ be infinite sets of cardinalities $|X'| < |X|$.
If $\Fcal^X_{fin}$ has a $(k+1) \to k$ monotone compression scheme, then 
$\Fcal^{X'}_{fin}$ has a $k \to (k-1)$ monotone compression scheme. 
 
 \end{lemma}
 
 Once we show that, the theorem follows by noting that no infinite set has monotone compression of size 0 of its size 1 samples.
 
 \begin{proof}[Lemma~\ref{lem:DecreaseSize}]
Let $\eta$ be a decompression function for $\Fcal^X_{fin}$ such that for every $S \subset X$ of size $k+1$ there exists $S' \subset S$ of size $|S'| \leq k$ such that $\eta(S') \supseteq S$. 

Since $X'$ is infinite
the set $Y = \bigcup_{S \subset X' : |S| \leq k} \eta(S)$ is of the same cardinality as $|X'|$.
Since $|X| > |X'|$, there is therefore $x \in X$ that is not in $Y$.

Now, for every $S' \subset X$ of size $k$, the compression $S''$ of $S= \{x\} \cup S'$ must contain $x$ since otherwise $x \notin \eta(S'')$.
 Therefore $S'' \setminus \{x\}$ is a subset of $S'$ of size $k-1$ such that $\eta(S'') \supset S'$. So we have a compression of size $k-1$ for the $k$-size subsets of $Y$:
the compression is of the form $S' \mapsto S'' \setminus \{x\}$
 and the decompression is obtained by applying $\eta$
 and taking the intersection of the outcome with $Y$.\qed
 \end{proof}
\qed \end{proof}

Well known theorems from set theory \cite{Easton,Jech,Kunen} now imply,  under the assumption that ZFC is consistent:
\begin{corollary}
The existence (and size when it exists) of monotone compression schemes for the class $\Fcal^{\reals}_{fin}$ 
is independent of ZFC set theory.
\end{corollary}

\begin{corollary} \label{cor:EMX independence}
The EMX-learnability of (and learning rates when it is learnable) of $\Fcal^{\reals}_{fin}$ with respect to the class of all probability distributions over the real line that have countable support is independent of ZFC set theory.
\end{corollary}

In terms of the Continuum Hypothesis in set theory, Theorems \ref{thm_compressionUB} and  \ref{thm_compressionLB} imply that $\Fcal^{\reals}_{fin}$ has a $2$-size monotone compression scheme iff $|\reals|=\aleph_1$. Hence we obtain a combinatorial statement that is equivalent to the Continuum Hypothesis. Similar such statements have been known previously, for example, the results of Sierpi\'{n}sky \cite{Sierp} (see also \cite{Erdos}) on decompositions  of the Euclidean plane and $\reals^3$, or the so-called Axioms of Symmetry of Freiling \cite{Freiling}. Let us remark without giving details that there is an intimate relationship between our results and Freiling's axioms. 

\subsection*{Imperfect reconstruction}

Consider the following natural generalization of the superset reconstruction game. 
{\em Alice is given $S\subseteq X$ of size $p$ and sends $S'\subseteq S$ of size $q<p$ to Bob. 
Bob needs to reconstruct a finite set $\eta(S')$ with $|\eta(S')\cap S|> q$.}
In other words, instead of finding a superset of $S$, Bob should find a set with at least one extra element from $S$. 

We now briefly discuss this ``imperfect'' reconstruction. 
For a set $X$, let $X^{(k)}$ denote the set of $k$-element subsets of $X$
and $X^{(<\omega)}$ denote the finite subsets of $X$.

\begin{definition}[$(p \to q \to r)$ Property]
\label{def:imperfect} Let $X$ be a set and let $p,q,r$ be integers $p\geq r\geq q>0$. We say that $X$ has the {\em $(p\to q\to r)$ property}, if there exist $\sigma: X^{(p)}\to X^{(q)}$ and $\eta:X^{(q)}\to X^{(<\omega)}$ such that for every $S\in X^{(p)}$,  
$$\sigma(S)\subseteq S \ \ \text{and} \ \ |\eta(\sigma(S))\cap S|\geq r.$$ The pair $(\sigma,\eta)$ is called a \emph{$(p\to q\to r)$-compression of $X$}.
\end{definition}

In this framework, Theorems \ref{thm_compressionUB} and \ref{thm_compressionLB} assert that
\begin{enumerate}
\item If $|X|\leq \aleph_{q-1}$ then $X$ has the $(p\to q\to p)$ property for every $p\geq q$.
\item If $|X|>\aleph_{q-1}$ then $X$ does not have the $((q+1)\to q \to (q+1))$ property.
\end{enumerate}
We can augment the picture by one more result:

\begin{theorem}\label{thm:imperfect}  If $X$ has $(p\to q \to (q+1))$ property then $X$ has cardinality $\leq \aleph_{p-2}$. 
\end{theorem}

\begin{proof} Let $X$ be as in the assumption. It is enough to show that $X$ has $(p\to (p-1)\to p)$ property.  For then Theorem \ref{thm_compressionLB} implies that $|X|\leq \aleph_{p-2}$.  

Assume that $(\sigma,\eta)$ is a $(p\to q\to (q+1))$-compression of $X$. Given $S\in X^{(p)}$, pick some $\alpha_S\in S$ such that 
$$\alpha_S\in \eta(\sigma(S))\setminus \sigma(S).$$ 
Let \[\sigma'(S):= S\setminus \{\alpha_S\}\,.\] 
Then $|\sigma'(S)|= p-1$. Furthermore,
$\alpha_S\in \eta(T)$ for some $q$-element subset $T$ of $\sigma'(S)$ (namely, the set $T= \sigma(S)$). Define $\eta': X^{(p-1)}\to X^{(<\omega)}$ by
\[\eta'(U):= U\cup \bigcup_{T\in U^{(q)}}\eta(T) . \]
Hence, $(\sigma',\eta')$ is a $(p\to (p-1)\to p)$-compression of $X$. 
\qed
\end{proof}

Note that the above results 
do not completely characterize the $(p\to q\to r)$ property in terms of the cardinality of $X$. For example, if $|X|=\aleph_1$ then $X$ has the  $(4\to 2\to 3)$ property, and it does not have the property  if $|X|=\aleph_3$. However, we do not know what happens in cardinality $\aleph_2$. 

\section{On the existence of a combinatorial dimension for EMX learning} 
 
\label{sec: dimension}

As mentioned above,
a fundamental result of statistical learning theory is the characterization of PAC learnability in terms of the VC-dimension of a class \cite{vapnik2015uniform,blumer1989learnability}. 
Variants of the VC-dimension similarly characterize other natural learning setups.
The Natarajan and Graph dimensions 
characterize multi-class classificiation
when the number of classes is finite.
When the number of classes is infinite, there is no known
analogous notion of dimension.
For learning real valued functions, the so called \emph{fat-shattering dimension} 
provides a characterization of the sample complexity~(\cite{kearns_schepire,AlonBCH97}).

It is worth noting that the aforementioned dimensions
also provide bounds on the sample complexity of the class.
For example, in binary classification the PAC-learning
sample complexity is $\Theta\bigl(\frac{d+\log(1/\delta)}{\eps}\bigr)$ where $d$
is the VC-dimension of the class~\cite{blumer1989learnability,Hanneke16}.
The Natarajan and Graph dimensions provide upper and lower bounds on the sample complexity of multi-class classification that are tight up to a logarithmic factor in the number of classes as well as $\log(1/\epsilon)$ (see e.g. \cite{Bendavid:1995aa}).

All of those notions of dimension are functions $D$ that map a class of functions $\Fcal$ to $\naturals \cup \{\infty\}$ and satisfy the following requirements:

\begin{description}
\item[Characterizes learnability:] A class $\Fcal$ is PAC learnable in the model if and only if $D(\Fcal)$ is finite.

\item[Of finite character:] $D(\Fcal)$ has a ``finite" character in the following sense: for every $d \in \naturals$ and a class $\Fcal$, the statement $D(\Fcal) \geq d$ can be demonstrated by a finite set of domain points and a finite set of members of $\Fcal$ (for real-valued function learning one also needs a finite set of rational numbers).
\end{description}

When it comes to EMX learnability, our results above imply that both the size of monotone compression for $\Fcal$ and the sample size needed for weak (say $\epsilon = 1/3, \delta=1/3)$ learnability of $\Fcal$ satisfy the first requirement for classes that are union bounded (and in particular, classes closed under finite unions), and also provide quantitative bounds on the sample complexity.

However, our results relating EMX learnability to the value of the continuum imply that no such notion of dimension can both characterize EMX learnability and satisfy the 
finite character requirement.

To formalize the notion of finite character, we use the first order language of set theory. Namely, first oder logic with the binary relation of membership $\in$.
{
Let ${\cal X}, {\cal Y}$ be variables.
A {\em property} is a formula $A({\cal X},{\cal Y})$ with the two free variables ${\cal X},{\cal Y}$.
A \emph{bounded formula} {$\phi$} is a first order formula in which all the 
quantifiers are of the form $\exists x \in {\cal X}, \forall x \in {\cal X}$
or $\exists y \in {\cal Y}, \forall y \in {\cal Y}$.}

\begin{definition}[Finite Character] 
A property $A({\cal X,Y})$ is 
a \emph{finite character property} 
if there exists a bounded formula $\phi({\cal X},{\cal Y})$
so that ZFC proves that $A$ and $\phi$ are equivalent.
%
\end{definition}

The intuition is that a finite character property can be checked by probing only into elements of  ${\cal X},{\cal Y}$, using the language of set theory.
We think of ${\cal X}$ as corresponding to the domain $X$
and ${\cal Y}$ as corresponding to the class $\Fcal$.
Given a model of set theory, the truth value
of a finite character property can be determined
for every specific choice of $X$ and $\Fcal \subseteq 2^X$.

%
{ The next straightforward claim describes the behavior
of finite character properties in different models.

\newcommand{\X}{\cal X}
\newcommand{\Y}{\cal Y}

\begin{claim} 
Let $A(\X,\Y)$ be a finite character property.
Let $M_0, M_1$ be models of set theory
and $M_1$ be a submodel of $M_0$.
If $X$ and $\Fcal \subseteq 2^X$ are the same in both models\footnote{That is, both the domain set $X$ and the class $\Fcal$ are members of both models and contain the same sets in both models.},  
then $A(X, \Fcal)$ holds in $M_0$ if and only if it holds in $M_1$.
\end{claim}}

Observe that for every integer $d$ 
the property 
``VC-dimension$(\Fcal) \geq d$'' is a finite character property.
It requires only a very weak formula; it can be expressed using only existential quantification into $X$ and $\Fcal$, since 
{its truth value is determined by a finite number of elements of $X$ and $\Fcal$.
Recall that PAC learnability is characterized by VC dimension
(we apply our measurability assumption to 
the definition of PAC learnability as well):}

\begin{theorem}
{For every integer $d$,
there exists integers $m,M$ so that 
for every set $X$ and $\Fcal \subseteq 2^X$
if VC-dimension$(\Fcal)$ is at most $d$ then the sample complexity
of $(1/3,1/3)$-PAC learning $\Fcal$ is at most $M$ and if
VC-dimension$(\Fcal)$ is more than $d$ then the sample complexity of $(1/3,1/3)$-PAC
learning $\Fcal$ is at least $m$.
The integers $m,M$ tend to $\infty$ as $d$ tends to $\infty$.}
\end{theorem}

{The theorem and the finite character of the statements ``VC-dimension$(\Fcal) \geq d$'' imply that PAC learnability cannot change 
between two models satisfying the assumption of the above Claim; 
loosely speaking,} PAC learnability does not depend on the specific model of set theory
that is chosen (under our measurability assumption).

On the other hand, we have seen that 
{EMX learnability heavily depends on the cardinality of the continuum. }
As a corollary, we obtain the following: 

\begin{theorem}  {There is some constant 
$c>0$ so that the following holds.
Assuming ZFC is consistent, there is no finite character property $A$
so that for some integers $m,M > c$ 
for every set $X$ and $\Fcal \subseteq 2^X$
if $A(X,\Fcal)$ is true then the sample complexity
of $(1/3,1/3)$-EMX learning $\Fcal$ is at most $M$
and if it is false then it is at least $m$.}
\end{theorem}

\begin{proof} Let $M_0$, $M_1$ be models of set theory such that:
\begin{enumerate}
\item $M_1$ is  a submodel of $M_0$.

\item The set of natural numbers and the set of reals (and therefore also the set of all finite subsets of real numbers) are the same in both models.
\item $M_0 \models (2^{\aleph_0}=\aleph_1)$. That is, $M_0$ satisfies the continuum hypothesis.
\item $M_1 \models (2^{\aleph_0}> \aleph_{\omega})$.

\end{enumerate}
Such models are known to exist (see, e.g., Chapter 15 of \cite{Jech}).

Consider the set $X$ of real numbers and the family $\Fcal^{\reals}_{fin}$. 
{By the above Claim}, 
 the truth value of every finite character property 
 is the same in both models.  
However, Corollary~\ref{cor:main} and Theorem~\ref{thm_compressionLB} implies that in $M_0$ the {class $\Fcal^{\reals}_{fin}$ is EMX learnable with 
a constant ($c-1$) number of samples, while in $M_1$
no finite number of samples suffices for learning it.}
\qed
\end{proof}

\section{Future research}
\label{sec:futRes}
\subsubsection{A stronger equivalence between compression and learnability?}
It would be interesting to determine whether monotone compression and EMX-learnability
are equivalent also without assuming that $\Fcal$ is closed under finite unions: 
\begin{question}
Does learnability imply monotone compression without assuming that $\Fcal$ is
union bounded? 
\end{question}
As mentioned in Section~\ref{sec:Intro}, this is tightly related to the open problem of existence of \emph{proper} L-W compression schemes:
does it follow from having a finite VC-dimension?


\subsubsection{Variants of the finite superset reconstruction game.}
It would also be interesting to study  
variants of the finite superset reconstruction game (Definition \ref{def: AB}). One such variant is given in Definition \ref{def:imperfect}. In this setting, Theorems~\ref{thm_compressionUB} -- \ref{thm:imperfect} provide some connections between the $(p\to q\to r)$ property and the cardinality of $X$. 
It would 
be nice to exactly characterize the $(p\to q\to r)$ property in terms of $|X|$ for general triples $p,q,r$. 

Other variants can be obtained, for example, by allowing Bob to reconstruct a countable subset of $X$ (or a subset of other prescribed cardinality). This version is more closely related to the aforementioned Axioms of Symmetry of  Freiling \cite{Freiling}. Similarly, Alice could be allowed to compress to an infinite subset of a given cardinality. 



\paragraph{Acknowledgements.} We thank David Chodounsk\'{y}, Steve Hanneke, Radek Honz\'{\i}k, 
and Roi Livni  for useful discussions. Part of this research was carried out while the first and third authors are staying at the Berkeley Simons Institute at the Foundations of Machine Learning program. We thank the Simons institute for that support.

\bibliographystyle{plain}
\bibliography{refs}

\section{Proof of Theorem~\ref{thm:compr_imply_EMX}}

\begin{proof} 

Given an input sample $S=(x_1, \dots x_m)$, and $A \in {[m]^{\leq k} := \cup_{i\leq k}[m]^i}$, 
let ${S_A}$ denote the subsample indexed by $A$: 
${S_A =(x_{A(i)}: i \leq \lvert A\rvert)}$ and, and let 
$S_{\bar A}$ denote the complementing subsample: 
$S_{\bar A} = \bigl(x_j : j \notin \{A(i) : i\leq \lvert A\rvert\}\bigr)$. 
Consider the random functions $\eta[S_A]$ for $A \in {[m]^{\leq k}}$. 
Evaluate each of these random functions on the remainder of the sample $S_{\bar A}$, and output the one with largest empirical mean. Formally, for a function $f$ and a sequence $T = (x_1,\ldots,x_k)$,
the empirical mean of $f$ with respect to $T$ is $\ex_T(f)= \frac{1}{k}\sum_{i \in [k]} f(x_i)$. 
On any input sample $S$, the learning algorithm outputs any {$h$} so that
$${h} \in \arg\max \big\{ \ex_{S_{\bar A}}(\eta[S_A]) : A \subseteq [m] , |A| \leq g(|S|) \big\}.$$

\begin{lemma}
For every $\delta \in (0,1)$ and a sample size $m$, 

\[\Pr_{S \sim P^m} \left[ \exists A \subseteq {[m]^{\leq k}}, |A| \leq {k} : \left| \ex_{S_{\bar A}}(\eta[S_A])  - \ex_P(\eta[S_A] )\right| > \alpha  \right] \leq \delta\]
with
\[\alpha = \sqrt{\frac{{k \ln (2m)} + \ln(1/\delta)}{2(m -{k})}}.\]
\end{lemma}

\begin{proof}
By the measurability of $\eta$ with respect to $P$ and the i.i.d. generation of $S$, for every $S_A$, the expected value $\ex_P(\eta[S_A]) $ can be approximated by the subsample $S_{\bar A}$, and Hoeffding inequality can be applied to get for every $ \delta \in (0,1)$:
\[\Pr_{S \sim P^m} \left[ \left| \ex_{S_{\bar A}}(\eta[S_A])  - \ex_P(\eta[S_A]) \right| > \sqrt{\frac{\ln(1/\delta)}{2(m -{k})}}\right] \leq \delta.\]
The lemma follows by noting that given a sample $S$ of size $m$, there are $\sum_{i=0}^{{k}} {m}^{i} \leq {2m^k}$ sequences $A$ and applying the union bound over them.
\qed \end{proof}

Applying the assumption that $\eta$ is a monotone compression for $\Fcal$, we get
that for every $h \in \Fcal$ and every $S$ there exists $\eta[S_A] \in \Fcal$ for which $\ex_P(\eta[S_A]) \geq \ex_P[h]$. 
Finally, fix some $h^{\star}$ in $\arg\max \{\ex_P(f): f \in \Fcal \}$. {By constraining $\alpha \leq \eps/2$ 
and setting the failure probability to $\delta/2$,} 
we can assume (with probability $1-\delta$) that the sample $S$ 
is also an $\epsilon/2$-accurate estimator of $\ex_P(h^{\star})$ 
(namely, that $|\ex_S(h^{\star}) - \ex_P(h^{\star})| \leq \epsilon/2$) . 

The $(\epsilon, \delta)$-success of the learning algorithm now follows once $m$ is large enough to render $\sqrt{\frac{{k \ln (2m)} + \ln(\new{2}/\delta)}{2(m -{k})}} \leq \epsilon/2$.
\qed \end{proof}

\end{document}